\documentclass[11pt,legal]{article}
\usepackage{theorem,ifthen,algorithm,algorithmic}
\usepackage{amssymb,amsfonts,amsmath,latexsym,dsfont}
\usepackage{tikz,pgflibraryplotmarks}
\usepackage{fullpage}
\usepackage{graphicx}
\usepackage{mathrsfs}

\graphicspath{{./img/}}

\usepackage{boxedminipage}

\usepackage{pgf,pgfarrows} 
\usepackage{subfigure} 

\newcommand{\FrameboxA}[2][]{#2}
\newcommand{\Framebox}[1][]{\FrameboxA}

\newcommand{\hf}{{\frac 12}}

\newcommand{\bfA}{{\bf A}}
\newcommand{\bfB}{{\bf B}}
\newcommand{\bfC}{{\bf C}}

\newcommand{\bfI}{{\bf I}}
\newcommand{\bfJ}{{\bf J}}
\newcommand{\bfK}{{\bf K}}
\newcommand{\bfL}{{\bf L}}

\newcommand{\bfW}{{\bf W}}
\newcommand{\bfX}{{\bf X}}
\newcommand{\bfY}{{\bf Y}}
\newcommand{\bfZ}{{\bf Z}}

\newcommand{\bfc}{{\bf c}}

\newcommand{\bfe}{{\bf e}}
\newcommand{\bfh}{{\bf h}}

\newcommand{\bfx}{{\bf x}}
\newcommand{\bfy}{{\bf y}}

\newcommand{\bfw}{{\bf w}}
\newcommand{\bfz}{{\bf z}}

\newcommand{\diag}{{\sf{diag}}\,}

\newcommand{\R}{\ensuremath{\mathds{R}}}

\newtheorem{lemma}{Lemma}

\newtheorem{proof}{Proof}
\newtheorem{example}{Example}
\usepackage[colorlinks=true]{hyperref}
\hypersetup{
citecolor=blue,
urlcolor=blue,
pdftitle={Stable Architectures for Deep Neural Networks},
pdfauthor={Eldad Haber and Lars Ruthotto},
pdfkeywords={Machine Learning, Deep Neural Networks, Dynamic Inverse problems, PDE-Constrained Optimization, Parameter Estimation, Image Classification}
}

\usepackage{authblk}

\author[1,3]{Eldad Haber}
\author[2,3]{Lars Ruthotto}
\affil[1]{Department of Earth and Ocean Science, The University of British Columbia, Vancouver, BC, Canada, (\url{haber@math.ubc.ca})}
\affil[2]{Emory University,
Department of Mathematics and Computer Science, Atlanta, GA, USA
(\url{lruthotto@emory.edu})}
\affil[3]{Xtract Technologies Inc., Vancouver, Canada, (\url{info@xtract.tech})}

\title{Stable Architectures for Deep Neural Networks }

\begin{document}

\maketitle

\begin{abstract}
    Deep neural networks have become invaluable tools for supervised machine learning, e.g., classification of text or images. While often offering superior results over traditional techniques and successfully expressing complicated patterns in data, deep architectures are known to be challenging to design and train such that they generalize well to new data. Critical issues with deep architectures are numerical instabilities in derivative-based learning algorithms commonly called exploding or vanishing gradients.  In this paper, we propose new forward propagation techniques inspired by systems of Ordinary Differential Equations (ODE) that overcome this challenge and lead to well-posed learning problems for arbitrarily deep networks.

     The backbone of our approach is our interpretation of deep learning as a parameter estimation problem of nonlinear dynamical systems. Given this formulation, we analyze stability and well-posedness of deep learning and use this new understanding to develop new network architectures. We relate the exploding and vanishing gradient phenomenon to the stability of the discrete ODE and present several strategies for stabilizing deep learning for very deep networks. While our new architectures restrict the solution space, several numerical experiments show their competitiveness with state-of-the-art networks.
\end{abstract}



\section{Introduction}
\label{sec:intro}

In this work, we propose new architectures for Deep Neural Networks (DNN) and exemplarily show their effectiveness for solving supervised Machine Learning (ML) problems; for a general overview about DNN and ML see, e.g.,~\cite{moller1993scaled,friedman2001elements,abu2012learning,GoodfellowEtAl2016} and reference therein. 
We consider the following classification problem: Assume we are given \emph{training data} consisting of $s$ feature vectors, $\bfy_1,\ldots,\bfy_s \in \R^{n}$, and label vectors, $\bfc_1, \ldots, \bfc_s \in \R^m$, whose $k$th components represent the likelihood of an example belonging to class $k$.  The goal is to \emph{learn} a function that approximates the data-label relation on the training data and generalizes well to similar unlabeled data. Our goals in this work are to highlight the relation of the learning problem to dynamic inverse problems, analyze its stability and ill-posedness for commonly used architectures, and derive new architectures that alleviate some of these difficulties for arbitrarily deep architectures.

We are particularly interested in deep learning, i.e., machine learning using neural networks with many hidden layers. DNNs have been successful in supervised learning, particularly when the relationship between the data and the labels is highly nonlinear; see, e.g., \cite{krizhevsky2012imagenet,hinton2012deep,bengio2009learning,lecun2015deep} and reference therein. Their depths (i.e., their number of layers) allow DNNs to express complex data-label relationships since each layer nonlinearly transforms the features and therefore effectively filters the information content.

Given the training data $(\bfy_1 ,\bfc_1), \ldots (\bfy_s ,\bfc_s)$, an inverse problem needs to be solved in order to \emph{train} a given network architecture. This problem, also called the \emph{learning problem}, aims at finding a parameterization of the DNN that explains the data-label relation and generalizes well to new unlabeled data. 
Clearly, using deeper network architectures increases the capacity of the network but also the dimensionality, and thus the computational complexity, of the parameter estimation problem.
Additionally, more labeled data is required to calibrate very deep networks reliably. Therefore, despite the fact that neural networks have been used since the early 70's, deep learning has only recently revolutionized many applications fueled by advances in computational hardware and availability of massive data sets.

Well-known sources of difficulty in deep learning are the dimensionality and non-convexity of the associated optimization problem. 
Traditionally, stochastic gradient descent methods have been used~\cite{RobbinsMonro1951}. It has been observed that the performance of the DNN can be highly dependent on the choice of optimization algorithm and sample size~\cite{byrd2012sample,bottou2016optimization}. Furthermore, it has been noted that some optimization algorithms yield DNNs that generalize poorly to new unlabeled data~\cite{KeskarEtAl2016}.

Additional difficulties in deep learning stem from instabilities of the underlying forward model, most importantly, the propagation of features through the DNN. As has been shown in~\cite{bottou2016optimization}, the output of some networks can be unstable with respect to small perturbations in the original features.
A related problem is the observation of vanishing or exploding gradients~\cite{BengioEtAl1994}.
These results are unsettling since  predictions made by networks with unstable forward propagation are very sensitive to small perturbations of the input features (as is common, e.g., in adversarial attacks), which may render the network useless in practice.

The main goal of this work is to gain new insight into the stability of the forward propagation and the well-posedness of the learning problem summarized in the following two questions:
\begin{enumerate}
\item
Given a network architecture and parameters obtained by some optimization process, is the forward propagation problem well-posed? 
\item Is the learning problem well-posed? In other words, given sufficient training are there parameters such that the DNN generalizes well or can generalization be improved by adding appropriate regularization?
\end{enumerate}

The first question is important because, while it may be possible to fit the training data even for unstable forward propagation models, the trained network is unlikely to generalize. In other words, small deviations in the data, e.g., due to noise, may be drastically amplified by the forward propagation, resulting in incorrect labels.
We show that the forward problem can be thought of as a discretization of an Ordinary
Differential Equation (ODE). Therefore, the stability of the network corresponds to the stability of its underlying ODE. Based on this observation we develop stability criteria for a simplified version of the commonly used Residual Network (ResNet) architecture~\cite{he2016deep} and develop {\em new network architectures} that are ensured to be stable and lead to well-posed learning problems.

The paper is organized as follows. In Section~\ref{sec:mathForm} we give a brief mathematical derivation of the deep learning problem illustrated using the ResNet architecture. 
 In Section~\ref{sec:stability} we analyze the stability of the simplified ResNet forward propagation and the well-posedness of the resulting learning problem.
Our analysis and examples suggest stability criteria, and as a result, in Section~\ref{sec:newArchitectures}, we propose three new stable architectures.
In Section~\ref{sec:optim} we propose regularization functions favoring smoothness of parameters and multi-level ideas to initialize deep networks.
In Section~\ref{sec:experiments} we demonstrate the effectiveness of our new architectures on a number of small-scale model problems and explore their effectiveness for solving image classification problems. Finally, in Section~\ref{sec:summary} we summarize the paper.

\section{Mathematical Formulation of the Deep Learning Problem} 
\label{sec:mathForm}
In this section we briefly describe three main ingredients of  deep learning  relevant to our work; for a comprehensive introduction see, e.g.,~\cite{moller1993scaled,friedman2001elements,abu2012learning,GoodfellowEtAl2016}.
First, we outline \emph{forward propagation} techniques, which transforms the input features in a nonlinear way to filter their information. Second, we describe \emph{classification}, which predicts the class label probabilities using the features at the output layer (i.e., the output of the forward propagation). Finally, we formulate the learning problem, which aims at estimating parameters of the forward propagation and classification that approximate the data-label relation.

For notational convenience we stack the training features and labels row-wise into matrices $\bfY_0 = [\bfy_1, \bfy_2, \ldots, \bfy_s]^\top \in \R^{s \times n}$ and $\bfC = [\bfc_1, \bfc_2, \ldots, \bfc_s]^\top \in \R^{s \times m}$.

To exemplify our discussion of forward propagation we consider a simplified version of the Residual Neural Network (ResNet)~\cite{he2016deep} model that has been very successful in classifying images using deep network architectures; see~\cite{GoodfellowEtAl2016} for other options.
In ResNets, the forward propagation of the input values, $\bfY_0 \in \R^{s \times n}$, through a network consisting of $N$ layers is given by
\begin{equation}
\label{nn}
\bfY_{j+1} =\bfY_j + h \sigma(\bfY_j \bfK_j + b_j)   \quad \text{ for }   \quad j=0,\ldots,N-1.
\end{equation}
The propagation in Eq.~\ref{nn} is parametrized by the nonlinear activation function $\sigma : \R^{s\times n} \to \R^{s \times n}$ and affine transformations represented by their weights, $\bfK_0, \bfK_1, \ldots, \bfK_{N-1} \in \R^{n\times n}$, and biases,  $b_0, b_1, \ldots, b_{N-1} \in \R$.
We augmented the original formulation in~\cite{he2016deep} by the parameter $h>0$ in order to increase the stability of the forward propagation and allow for a continuous interpretation of the process; see also Section~\ref{sec:stability}.
The values  $\bfY_1,\ldots,\bfY_{N-1}$ are also called \emph{hidden} layers and $\bfY_N$ is called the \emph{output} layer.
The activation function is applied element-wise and is typically (piecewise) smooth and monotonically non-decreasing. As two commonly used examples,  we consider the hyperbolic tangent and the Rectified Linear Unit (ReLU) activations
\begin{equation*}
	\sigma_{\rm ht}(\bfY) = \tanh(\bfY) \quad \text{ and } \quad \sigma_{\rm ReLU}(\bfY) = \max(0,\bfY). 
\end{equation*}

The class label probabilities are predicted using the values at the output layers, $\bfY_N$, a \emph{hypothesis function} 
 $\bfh(\bfY_N \bfW +\bfe_s \mu^\top)$, and its associated weights, $\bfW \in \R^{n \times m}$, and bias, $\mu\in\R^m$. Here $\bfe_k \in\R^k$ denotes the $k$-dimensional vector of all ones.  
For Bernoulli variables (i.e., $\bfC \in \{0,1\}^{s\times m}$) it is natural to consider the logistic regression function
\begin{equation}\label{eq:logReg}
	\bfh(\bfx) = \exp(\bfx) ./ (1 + \exp(\bfx)),
\end{equation}
where the exponential and the division are applied element-wise. For Multinomial distributions we use the softmax function
\begin{equation}\label{eq:softmax}
	\bfh(\bfX) = \exp(\bfX) ./ (\exp(\bfX)\bfe_m).
\end{equation}

The learning problem aims at estimating the parameters of the forward propagation (i.e., $\bfK_0, \bfK_1, \ldots, \bfK_{N-1}$ and $b_0, b_1,\ldots, b_{N-1}$) and the classifier ($\bfW$ and $\mu$) so that the DNN accurately approximates the data-label relation for the training data \emph{and} generalizes to new unlabeled data.
As we show below, the learning problem can be cast as a dynamic inverse problem, which provides new opportunities for applying theoretical and computational techniques from parameter estimation to deep learning problems. We phrase learning as an optimization problem
\begin{eqnarray}
	\label{opt}
\min
&&  \frac{1}{s} S(\bfh(\bfY_N \bfW +\bfe_s \mu^\top) ,\bfC) + \alpha R(\bfW,\mu,\bfK_{0,\ldots,N-1}, b_{0,\ldots,N-1})\nonumber\\
\label{st}
{\rm s.t. } &&
\bfY_{j+1} = \bfY_j + h \sigma(\bfY_j \bfK_j + b_j),\quad j=0,1,\ldots,N-1,
\end{eqnarray}
where the loss function $S$ is convex in its first argument and measures the quality of the predicted class label probabilities, the convex regularizer $R$ penalizes undesirable (e.g., highly oscillatory) parameters, and the parameter $\alpha >0$ balances between minimizing the data fit and regularity of the parameters.
A simple example for a loss function is the sum-of-squared difference function $S(\bfC_{\rm pred},\bfC)  = \frac12 \|\bfC_{\rm pred} - \bfC \|^2_F$.  Since our numerical experiments deal with classification we use cross entropy loss functions. 
Choosing an "optimal" regularizer, $R$, and regularization parameter, $\alpha$, is both crucial and nontrivial. 
Commonly Tikhonov regularization~\cite{ehn1,hansen,vogelbook}, also referred to as weight decay, has been used~\cite{goodfellow2009measuring}, although, other possibilities enforcing sparsity or other structure have been proposed~\cite{ng2004feature}.
We introduce novel regularization functions in Section~\ref{sec:optim}. For simplicity, we assume that a suitable value of $\alpha>0$ is chosen by the user or that it is done dynamically as suggested in~\cite{burger2006}. 

There are numerous approaches to solving the learning problem. In this work we use a simple block coordinate descent method to demonstrate the properties of the forward propagation. 
Our method alternates between updating the parameters of the classifier $(\bfW, \mu)$ fixing the current value of the propagated features $\bfY_N$ and then updating the parameters of the forward propagation while keeping the updated weights of the classifier fixed.
The first problem is typically convex and the latter problem is generally non-convex due to the forward propagation process.
Both steps are based on subsampling the training data. To this end note that most common loss functions can be written as a sum over all examples, i.e., 
\begin{eqnarray*}
	\frac1s S(\bfh(\bfY_N \bfW +\bfe_s \mu^\top) ,\bfC)  & =  \frac1s \sum_{i=1}^s S(\bfh(\bfy_i^\top \bfW +\mu^\top) ,\bfc_i^\top) \\
	& \approx \frac{1}{|\mathcal{T}|} \sum_{i\in \mathcal{T}} S(\bfh(\bfy_i^\top \bfW + \mu^\top) ,\bfc_i^\top),
\end{eqnarray*}
where $\mathcal{T} \subset \{1,2,\ldots,s\}$ is a randomly chosen set updated in each iteration of the block coordinate descent method. 
The size of the batches is a parameter in our algorithm whose choice depends on the size and complexity of the problem and resource considerations.
In the following, we assume the sample size is constant; for adaptive selection of sample size see, e.g.,~\cite{byrd2012sample}.
In each iteration of the block coordinate descent scheme, our algorithm approximately solves the resulting classification problem using an Newton-PCG method, i.e., an inexact Newton method that uses a Preconditioned Conjugate Gradient (PCG) method to determine a search direction (see, e.g., \cite{nw,hest,saad2}). 
Subsequently, the weights of the forward propagation are updated using a Gauss-Newton-PCG method. 
Note that gradient and approximated Hessian computations require matrix-vector products with the derivative matrices of the values of the output layer, $\bfY_N$, with respect to $\bfK_{0,\ldots,N-1}$ and $b_{0,\ldots,N-1}$. 
The matrix-vector products with the derivative matrix can be computed without its explicit construction through forward and backward propagation, respectively; see also~\cite{MartensSutskever2012}. However, this requires storing (or recomputing) the values at the output layers; see, for instance, Section~\ref{sub:derVerlet} for derivative computation. 
Therefore, as also suggested in~\cite{ByrdEtAl2011}, we subsample $\mathcal{T}$ further to reduce the cost of the Hessian matrix-vector products in our PCG scheme.

Our implementation includes computing the validation error in each iteration of the block coordinate descent method. The final output of our algorithm are the parameters that achieve the lowest validation error. 

\section{Stability and well-posedness of the forward propagation}
\label{sec:stability}
In this section we analyze the stability of the ResNet forward problem~\ref{nn} and illustrate why some choices of transformation weights may generate instabilities or prohibit effective learning altogether.

It is well-known that any parameter estimation problem requires a well-posed forward problem, i.e., a problem whose output is continuous with respect to its input. 
For example, practical image classification algorithms need to be robust against noisy or slightly shifted input images. Ill-posedness of the forward problem implies that even if the estimated parameters lead to a small training error they will probably fail or will do poorly on a perturbation of that data. In other words a network whose forward propagation is ill-posed will generalize poorly. Thus,  well-posedness of the forward propagation is a necessary condition to obtain DNNs that generalize well.

The following discussion also gives new perspectives on two well-known phenomena in the deep learning community: Vanishing and exploding gradients; see, e.g.,~\cite{BengioEtAl1994}. 
These phenomena refer to the gradient of the objective function in Eq.~\ref{opt} and pose a severe challenge for very deep architectures.
 Note that the gradient represents the sensitivity of the output with respect to a perturbation in the input. Thus, an exploding gradient implies that the output 
  is unstable with respect to the input. Similarly a vanishing gradient implies that the output 
  is insensitive with respect to the input. Clearly both cases prohibit effective training, but more importantly, may not provide DNNs that generalize well.

To understand the phenomena we consider a simplified version of the forward propagation in ResNets given in Eq.~\ref{nn}. As pointed out in \cite{haber2017learning} the forward propagation can be seen as an explicit Euler discretization of the nonlinear Ordinary Differential Equation (ODE)
\begin{equation}
\label{ode}
\dot \bfy(t) = \sigma(\bfK^{\top}(t) \bfy(t) + b(t)), \quad \text{ with } \quad \bfy(0) = \bfy_0,
\end{equation}
over a time interval $t = [0,T]$. The final time $T>0$ and the magnitude of $\bfK(t)$ control the depth of the network.
The ODE is stable if $\bfK(t)$ is changing sufficiently slow and 
\begin{equation}\label{eq:condK1}
	\max_{i=1,2,\ldots,n} Re(\lambda_i (\bfJ(t)))\le 0, \quad \forall t \in [0, T],
\end{equation}
where $Re(\cdot)$ denotes the real part, $\lambda_i(\bfJ(t))$ is the $i$th eigenvalue of the Jacobian of the right hand side in Eq.~\ref{ode}, denoted by  $\bfJ(t) \in \R^{n \times n}$. 
A more accurate statement that uses kinematic eigenvalues of the Jacobian can be found in \cite{amr}. Here, the Jacobian is
\begin{eqnarray}
	\bfJ(t) &= \left(\nabla_{\bfy}\left( \sigma(\bfK(t)^\top \bfy + b(t)) \right)\right)^\top \nonumber \\
	       & = {\rm diag}\left(\sigma'(\bfK(t)^\top \bfy + b(t))\right)\, \bfK(t)^\top.	\label{eq:eigJ}
\end{eqnarray}
Since the activation function $\sigma$ is typically monotonically non-decreasing, i.e., $\sigma'(\cdot) \geq 0$, Eq.~\ref{eq:condK1} is satisfied if $\bfK$ changes sufficiently slowly and 
\begin{equation}
\label{condK}
\max_{i=1,2,\ldots,n} Re(\lambda_i (\bfK(t))) \le 0,  \quad \forall t \in [0,T].
\end{equation}
Controlling the smoothness of $\bfK$ can be done by regularization as described in Section~\ref{sec:optim}.
To ensure the stability of the overall discrete forward propagation, we also require the discrete version of the ODE to have a sufficiently small $h$ as summarized in the following well-known lemma.
\begin{lemma}[Stability of Forward Euler Method]
The forward propagation in Eq.~\ref{nn} is stable if
\begin{equation}
\label{dstab}
 \max_{i=1,2,\ldots,n} | 1+ h  \lambda_i(\bfJ_j)|   \le 1, \quad \forall j = 0,1,\ldots, N-1.
 \end{equation}
\end{lemma}
\begin{proof}
See, e.g.,~\cite{ascherBook, apbook} for the proof of stability criteria for the forward Euler method.
\end{proof}

The above discussion suggests that the stability of the continuous forward propagation in Eq.~\ref{condK} 
and its discrete analog Eq.~\ref{dstab} 
need to be added to the optimization problem~\ref{opt} as constraints.
Otherwise, one may obtain some transformation weights, $\bfK_0, \bfK_1, \ldots, \bfK_{N-1}$, that may fit the  training data but generate an unstable process.
As discussed above these solutions cannot be expected to generalize well for other data.

We illustrate the stability issues in ResNet using a simple example. 
\begin{example}[Stability of ResNet]
	\label{exam1}
	For $s=3$ and $n=2$ we consider the forward propagation through a ResNet as given by Eq.~\ref{nn}.
	We consider three networks consisting of $N=10$ identical layers, i.e., on each layer we use the activation $\sigma_{\rm ht} = \tanh$, $h=0.1$, $b=0$, and a constant weight matrix. To illustrate the impact of the eigenvalues of the weight matrix on the propagation, we consider three ResNets parameterized by
	\begin{equation}\label{eq:KernelEx}
		\bfK_+ = \left( 
		\begin{array}{rr}
			 2 & -2 \\ 0 & 2
		\end{array}
		\right), \; \bfK_{-} = \left( \begin{array}{rr}
			 -2 & 0 \\ 2 & -2
		\end{array}	\right),
		 \;
		\bfK_{0} = \left( \begin{array}{rr}
			 0 & -1 \\ 1 & 0
		\end{array}
		\right),
	\end{equation}
	where $\lambda_1(\bfK_+) = \lambda_2(\bfK_+) = 2, \ \  \lambda_1(\bfK_{-}) =  \lambda_2(\bfK_{-}) =-2$ and $\lambda_1(\bfK_0) =\imath, \ \ \lambda_2(\bfK_0) = -\imath$.
	We consider the feature vectors $\bfy_1 = [0.1,0.1]^{\top},\ \bfy_2 = -\bfy_1,\ \bfy_3=[0, 0.5]^\top$. 
	After propagating the features through the layers, we illustrate the different propagations in Figure~\ref{fig1}. 
	We represent the values at the hidden layers as colored lines in the 2D plane where each color is associated with one feature vector. 
	To highlight the differences in the dynamics in all three cases, we also depict the force field using black arrows in the background. 
	This plot is often referred to as the \emph{phase plane diagram}.
	\begin{figure}
		\begin{center}
			\includegraphics[width=1\textwidth]{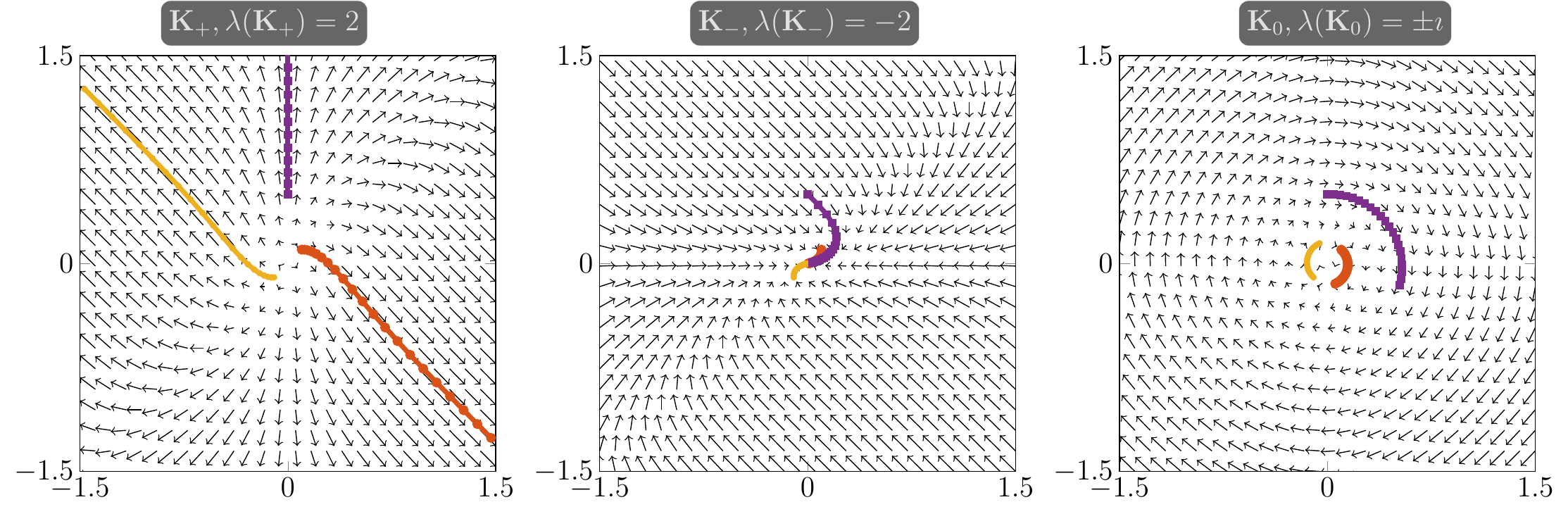}
		\end{center}
	\caption{Phase plane diagrams for ResNets with $N=10$ identical layers parameterized by the weight matrices in Eq.~\ref{eq:KernelEx} starting from three different input features. The values at hidden layers are indicated by colored lines and arrows depict the force field. Left: Due to its positive eigenvalues, the features diverge using $\bfK_+$ leading to an unstable forward propagation. Center: $\bfK_-$ yields a contraction that annihilates differences in the features and renders the learning problem ill-posed. Right: The antisymmetric matrix $\bfK_{0}$ leads to rotations, preserves distances between the features, and yields well-posed forward propagation and learning.  \label{fig1}}
	\end{figure}
	
	As can be seen in left subplot, the features diverge away from the origin and each other using $\bfK_+$. 
	Note that $\bfy_1$ and $\bfy_2$, which are close together initially, depart into opposite directions. This clearly suggests an unstable forward propagation that cannot be expected to generalize well. 
	In contrast to that, the center subplot shows that $\bfK_{-}$ yields  an accumulation point at the origin. 
	While the forward propagation satisfies Eq.~\ref{condK} and is thus stable, the learning problem, which requires inversion of this process, is ill-posed.
	Finally, the antisymmetric matrix $\bfK_0$ yields a rotation in the feature space, which preserves the distances between the features and leads to a stable forward propagation and a well-posed learning problem.
	
	Clearly the effects depicted here are more pronounced in deeper networks with more layers and/or larger values for $h$.
\end{example}

While the first case in Example~\ref{exam1} indicates that Eq.~\ref{condK} is a necessary condition for successful learning, the second case suggests that it is not sufficient.
In general, when $Re(\lambda(\bfK(t)))<0$ for most times and the network is deep (e.g., long time integration) differences in the initial feature vectors  decay. 
In other words, for all initial conditions $\bfy_0$ we have that $\bfy(t) \rightarrow 0$ as $t \to \infty$; compare with center plot in Figure~\ref{fig1}.
Hence, even though the forward problem is stable the inverse problem is highly ill-posed as it comparable to an inverse heat equation.
In these situations, the gradients of the objective function in Eq.~\ref{opt} will vanish since  small changes in the kernels will have no noticeable impact on the values of the outputs.

\bigskip

If we are inspired from the propagation of signals through neurons, then  a ResNet as in Eq.~\ref{st} with
$\max_i Re\left(\lambda_i (\bfK) \right) > 0$ for all $i$ consists of neurons that amplify the signal with no upper bound (which is not biological) and a ResNet with 
   $\max_i Re\left(\lambda_i (\bfK) \right)\ll 0$ can be seen as a lossy network.
   A moderately lossy network may be advantageous when the input is noisy, since it tends to decay high order oscillations. However, having too much signal loss is clearly harmful as it also annihilates relevant differences in the input features.

In summary, our discussion of ResNets illustrates that stable forward propagation and well-posed learning problems can be obtained for deep networks when
\begin{equation}\label{Kapprox0}
Re\left(\lambda_i (\bfK(t)) \right) \approx 0, \quad \forall i=1,2,\ldots,n, \; t \in [0,T].	
\end{equation}
In this case, the forward propagation causes only moderate amplification or loss and thus even deep networks preserve features in the input data and allow for effective learning.

\section{Stable Forward Propagation for DNNs}
\label{sec:newArchitectures}
Motivated by the discussion in the previous section, we introduce three new forward propagation methods that are stable for arbitrarily deep neural networks and lead to well-posed learning problems.
Our approaches, presented in Section~\ref{sub:antiSym} and \ref{sub:Hamiltonian}, use different means to enforce Jacobians whose eigenvalues have very small real part; see Eq.~\ref{Kapprox0}. The methods in Section~\ref{sub:Hamiltonian} are inspired by Hamiltonian systems and we propose leapfrog and Verlet integration techniques for forward propagation in Section~\ref{sub:symplectic}. Finally, we compute the derivative of the Verlet method using back propagation~\cite{williams1986learning} and discuss the relation between back propagation and the older and more general adjoint method~\cite{kalman1960contributions} in Section~\ref{sub:derVerlet}.

\subsection{Antisymmetric Weight Matrices}
\label{sub:antiSym}
Perhaps the simplest way to obtain a stable forward propagation is to construct force fields whose Jacobians are antisymmetric.
For example, consider the forward propagation
\begin{equation}
\label{antisymdisc}
\bfY_{j+1} = \bfY_j + h \sigma \left(\hf \bfY_j \left(\bfK_j - \bfK_j^{\top} - \gamma \bfI\right) + b_j\right), \; j=0,\ldots,N-1,
\end{equation}
where $\gamma\geq0$ is a small constant and $\bfI$ is an identity matrix. The resulting forward propagation can be seen as a forward Euler discretization of the ODE
\begin{equation*}
\dot{\bfy}(t) = \sigma\left(\hf (\bfK(t) - \bfK(t)^{\top} - \gamma \bfI) \bfy(t) + b(t)\right), \quad \forall t \in [0,T].	
\end{equation*}
Since $\bfK(t) - \bfK(t)^{\top}$ is anti-symmetric its eigenvalues are imaginary, which also holds for the Jacobian in Eq.~\ref{eq:eigJ}. Thus, the continuous ResNet parameterized by the antisymmetric weight matrix is stable and preserves information given an appropriate integration technique and sufficiently small time steps.

To stabilize the discrete forward propagation, which is based on the forward Euler method, we have added diffusion to the system in Eq.~\ref{antisymdisc}. The amount of diffusion depends on the parameter $\gamma\geq0$.  While small values of $\gamma$ might improve the robustness against noise in the feature vectors, too large values can render the learning problem ill-posed. Alternatively, more advanced time integration methods can be used with $\gamma=0$, i.e., without adding diffusion.

\subsection{Hamiltonian Inspired Neural Networks}
\label{sub:Hamiltonian}

Restricting the parameter space to antisymmetric kernels is only one way to obtain a stable forward propagation.
Alternatively we can recast forward propagation as a Hamiltonian system, which has the structure
\begin{equation*}
\dot \bfy(t) = - \nabla_{\bfz} H(\bfy,\bfz,t) \quad \text{ and } \quad
\dot \bfz(t) =  \nabla_{\bfy} H(\bfy,\bfz,t), \quad \forall t \in [0,T].
\end{equation*}
The function $H : \R^n \times \R^n \times [0,T] \to \R$ is the Hamiltonian. In our application  Hamiltonian systems for forward propagation are attractive due to their property to conserve rather than increase or dissipate the energy of the system; see~\cite[Ch.6]{ascherBook} for a general introduction. 

The Hamiltonian function $H$ measures the energy of the system, which in the autonomous case gets conserved. A simple approach can be derived from the Hamiltonian
$$ H(\bfy,\bfz) = \hf \bfz^{\top}\bfz + f(\bfy), $$
where $f : \R^n \to \R$ is at least twice continuously differentiable.
This leads to the ODE system 
$$ \dot \bfz(t) = -\nabla_{\bfy} f(\bfy(t)), \quad \dot \bfy(t) = \bfz(t), \quad \forall t \in [0,T]. $$
Eliminating $\bfz$ we obtain a second order system
$$ \ddot \bfy(t)= \nabla_{\bfy} f(\bfy(t)), \quad \forall t \in [0,T].$$

Inspired by the continuous version of the ResNet forward propagation given in Eq.~\ref{ode} we propose to use the following second order ODE
\begin{equation}
\label{secOrder}
\ddot \bfy(t)  = \sigma(\bfK^{\top}(t)\bfy(t) + b(t)), \quad \bfy(0) = \bfy_0, \quad \dot{\bfy}(0) = \dot{\bfy}_0,
\end{equation}
where in the following we assume that $\dot{\bfy}_0 = 0$.
The dynamical system in Eq.~\ref{secOrder} is stable for all weight matrices with non-positive real eigenvalues, i.e., that satisfy Eq.~\ref{condK}.
This constraint can either be added to the optimization problem~\ref{opt} or, similar to the previous section, be enforced by design. The latter approach can be obtained, e.g., using the parametrization
\begin{equation}
\label{negker}
\bfK(\bfC) = -\bfC^{\top}\bfC, \quad \text{ for some }\quad \bfC \in \R^{n\times n}.
\end{equation}
Note that in contrast to the antisymmetric model in Eq.~\ref{antisymdisc}, the negative definite model is based on a nonlinear parametrization which might lead to a more challenging optimization problem. Alternatively Eq.~\ref{condK} can be added as a constraint to the optimization problem~\ref{opt} requiring expensive eigenvalue computations.

\bigskip

The forward propagations Eqs.~\ref{nn}, \ref{antisymdisc}, or \ref{secOrder} require additional constraints on
$\bfK$ and are limited to square weight matrices. We therefore propose a network that is {\em intrinsically} stable and supports non-square weight matrices, i.e., 
a network whose forward propagation is stable independent of the choice of the weights. To this end, we introduce symmetry into the Hamiltonian system by defining
\begin{equation}
\label{hamiltonian}
\dot \bfy(t) = \sigma(\bfK(t) \bfz(t) + b(t))  \; \text{ and } \;
\dot \bfz(t) = -\sigma(\bfK(t)^{\top} \bfy(t) + b(t)).
\end{equation}
Note that Eq.~\ref{hamiltonian} can be seen as a special case of ResNet with an augmented variable $\bfz \in \R^m$ and the associated ODE
$$
 {\frac {\partial}{\partial t}} 
\left(\begin{array}{r} \bfy \\ \bfz \end{array}\right)(t) =  \sigma \left( 
\left(\begin{array}{rr} 
	0                & \bfK(t) \\ 
	-\bfK(t)^{\top}  & 0  \\
	\end{array} \right)
	\left(\begin{array}{r}
		 \bfy \\ \bfz 
		\end{array}\right) (t) + b(t) \right), \; \left(\begin{array}{r} \bfy \\ \bfz \end{array}\right)(0) = \left(\begin{array}{r}\bfy_0 \\ 0
\end{array}\right).
$$
This system is stable regardless of the spectrum or size of the matrices $\bfK(t)$ since the overall matrix in the linear transformation is antisymmetric. To ensure stability of the discrete version the time step size needs to be sufficiently small; see Lemma~\ref{dstab}.

\subsection{Symplectic Forward Propagation}
\label{sub:symplectic}
In this section we use symplectic integration techniques for solving the discrete versions of the Hamiltonian-inspired networks in Eq.~\ref{secOrder} and Eq.~\ref{hamiltonian}.
Symplectic methods have been shown to capture the long time features of Hamiltonian systems and thus provide a slightly different approach as compared to the forward Euler method used in ResNet~\ref{st}; we refer to~\cite[Ch.6]{ascherBook} for a detailed discussion.
For the second-order ODE in Eq.~\ref{secOrder} we recommend the conservative leapfrog discretization 
\begin{equation}
\label{LF}
\bfy_{j+1} = 
\left\{\begin{array}{ll}
	2\bfy_j  + h^2\sigma(\bfK_j \bfy_j + b_j), &  j=0\\
	2\bfy_j - \bfy_{j-1} + h^2\sigma(\bfK_j \bfy_j + b_j), &    j=1,2,\ldots,N-1
\end{array}\right..
\end{equation}  
To discretize the augmented network in Eq.~\ref{hamiltonian} we use a Verlet integration where for $j=0,1,\ldots, N-1$ we have
\begin{equation}
\label{varlet}
\bfz_{j+\hf} = \bfz_{j-\hf} - h \sigma(\bfK^{\top}_j \bfy_j + b_j) \; \text{ and } \;
 \bfy_{j+1} = \bfy_j + h \sigma(\bfK_{j} \bfz_{j+\hf} + b_{j}).
\end{equation}
Since both discretizations  are \emph{symplectic}, the respective Hamiltonians are preserved if the transformation weights are time invariant and the step size is sufficiently small.

We demonstrate the long time dynamics of the two new Hamiltonian-inspired forward propagation methods using a simple  example with two-dimensional features and identical layers.
\begin{example}
	Let $s = 2$ and $n=2$ and consider the features $\bfy_1=[0.1,0.1]^\top$ and $\bfy_2=[-0.1, -0.1]^\top$. 
	We consider networks with identical layers featuring $b=0$ and hyperbolic tangent as activation function. 
	For the leapfrog integration we use the matrix $\bfK = \bfK_-$ defined in Eq.~\ref{eq:KernelEx} (recall that $\lambda_1(\bfK) = \lambda_2(\bfK) = -2$) and a step size of $h=1$. To illustrate the Verlet integration (which can handle non-square kernels) we use a time step size of $h=0.1$ and
	\begin{equation*}
	\bfK_v = \left( \begin{array}{rr}
		2 &  1 \\ -1  &  2 \\ 0 & 1
	\end{array}
	\right).
	\end{equation*}
	To expose the short term behavior we use $N=500$ layers and to demonstrate the non-trivial long term characteristics we use $N=5,000$ layers. 
	We depict the phase plane diagrams for both networks and both depths in  Figure~\ref{fig2}.
	Even though we use identical layers with constant kernels the features neither explode nor vanish asymptotically.  
	\begin{figure}
	\begin{center}
		\includegraphics[width=\textwidth]{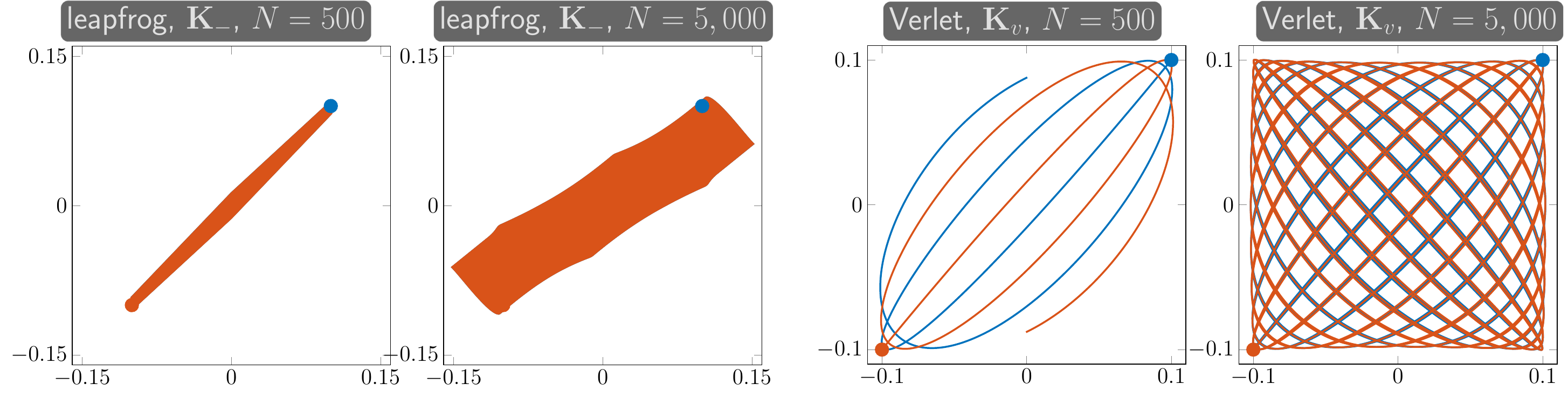}
	\end{center}
	\caption{Phase space diagrams for the forward propagation using the leapfrog and Verlet methods starting from $\bfy_1= [0.1,0.1]^\top$ (blue) and $\bfy_2=[-0.1,-0.1]^\top$ (red). For each network we show the short term ($N=500$) and long term ($N=5,000$) behavior for identical layers. In both cases non-trivial behavior can be observed even for constant weight matrices. (This figure is optimized for screen use.) \label{fig2}}
	\end{figure}
	\end{example}
We proposed three new ways of forward propagation to ensure stability for arbitrary numbers of layers. 
Among the three approaches, the Verlet method is the most flexible as it can handle non-square weighting matrices and does not require additional constraints on $\bfK$.

\subsection{Derivatives of the Verlet Method}
\label{sub:derVerlet}
We now discuss the computation of derivatives for the Verlet method. The derivatives of the ResNet are standard
and can be found, e.g., in~\cite{he2016identity}. Given the ResNet derivatives the antisymmetric model can be differentiated using chain rule and the time stepping. The leapfrog method can be treated in a similar way to usual ResNets and is therefore omitted.

Our presentation follows the standard approach used in machine learning known as \emph{ back propagation }\cite{he2016identity}.  However, we note that back propagation is a special case of the \emph{adjoint method}~\cite{bliss1919use} used in time-dependent optimal control; see, e.g.,\cite{CaoEtAl2003, BorzSchulz2012}.  The adjoint method is more general than back propagation as it can also be used compute the gradient of less ``standard'' time stepping methods. Thus, we discuss the back propagation method in the context of the adjoint method as applied to the Verlet method. 
We start by the usual sensitivity equation and then discuss how to efficiently compute matrix-vector products.

Exemplarily we show how to differentiate the values at the output layer in Eq.~\ref{varlet} with respect to the weight matrix $\bfK_k$ at an arbitrary layer $k$. For simplicity we assume that $\bfK$ is square.  Derivatives for non-square weight matrices and with respect to the bias, $b_k$, can be computed along the same lines. Clearly, the values of $\bfz_j$ and $\bfy_j$ at the hidden layers $j \leq k$ are independent of $\bfK_k$.  
Applying the chain rule to Eq.~\ref{varlet} we see that 
\begin{eqnarray}
		{\frac {\partial \bfz_{k+\hf}}{\partial \bfK_k}} &= h \diag\left(\sigma'(\bfK_k \bfy_k + b_k)\right) (\bfy_k \otimes \bfI) =:  \bfC_1  \label{dvarletjc1}
		\\
		{\frac {\partial   \bfy_{k+1}}{\partial \bfK_k}} &= - h \diag\left(\sigma'(-\bfK_{k}^\top \bfz_{k+\hf} + b_{k})\right)  (\bfI \otimes \bfZ_{k+\hf} + \bfK_k^\top \otimes \bfI)\nonumber \\
		&=:   \bfC_2, \label{dvarletjc2}
 \end{eqnarray}
where $\otimes$ denotes the Kronecker product and $\bfI$ is the $n\times n$ identity matrix.
We can now differentiate layer by layer, where for $k < j < N-1$ we obtain
\begin{eqnarray}
\label{dvarlet}
	{\frac {\partial \bfz_{j+\hf}}{\partial \bfK_k}} =&{\frac {\partial  \bfz_{j-\hf}}{\partial \bfK_k}} + h{\rm diag} \left( \sigma'(\bfK_j \bfy_j + b_j) \right) \bfK_j {\frac {\partial \bfy_{j}}{\partial \bfK_k}} \\
	{\frac {\partial   \bfy_{j+1}}{\partial \bfK_k}} =& {\frac {\partial \bfy_j}{\partial \bfK_k}} - h {\rm diag} \left(\sigma'(-\bfK_{j}^\top \bfz_{j+\hf} + b_{j}) \right) \bfK_j^\top {\frac {\partial  \bfz_{j+\hf}}{\partial \bfK_k}}.
 \end{eqnarray}
Combining equations~\ref{dvarletjc1}, \ref{dvarletjc2}, and~\ref{dvarlet} the derivatives can be written compactly as a block linear system
\begin{equation}
\label{bigsystem}
\hspace{-15mm}
\left(\begin{array}{rrrrrrr}
 \bfI          &               &               &            &               &         &        \\
               &  \bfI         &               &            &               &         &        \\
   -\bfI       &   \bfB_{k+1}  & \bfI          &            &               &         &        \\
               &   -\bfI       & \bfA_{k+1}    &  \bfI      &               &         &        \\
               &               &  \ddots       & \ddots     & \ddots        &         &      \\
               &               &               & -\bfI      & \bfA_N        & \bfI    &        \\
               &               &               &            & -\bfI         & \bfB_N  & \bfI     \\
\end{array}\right)
\left(\begin{array}{r}
{\frac {\partial \bfz_{k+\hf}}{\partial \bfK_k}} \\
{\frac {\partial \bfy_{k+1}}{\partial \bfK_k}} \\
\vdots \\
 \\
{\frac {\partial \bfz_{N+\hf}}{\partial \bfK_k}} \\
{\frac {\partial \bfy_{N+1}}{\partial \bfK_k}} \\
\end{array}\right) = 
\left(
\begin{array}{c}
\bfC_1\\
\bfC_2\\
   0  \\  
  \vdots \\
     \\
     \\
     0
\end{array} \right), 
\end{equation}
where 
\begin{eqnarray*}
\bfB_j & = -h  \diag\left( \sigma'(\bfK_j \bfy_j + b_j) \right) \bfK_j,\\
\bfA_j & = h \diag\left( \sigma'(-\bfK_j^{\top} \bfz_{j+\hf} + b_j) \right) \bfK_j^{\top}. 
\end{eqnarray*}

 The linear system in Eq.~\ref{bigsystem} is block triangular with  identity matrices on its diagonal and thus can be solved  explicitly using forward substitution.
 Such systems commonly arise in optimal control of time-dependent PDEs and in dynamic inverse problems. For more details
 involving notes on implementation see, e.g.,~\cite{rothauge2015numerical}.

In the optimal control literature the matrices ${\frac {\partial \bfz_{j+\hf}}{\partial \bfK_k}}$ and ${\frac {\partial \bfy_j}{\partial \bfK_k}}$ are often
referred to as the sensitivity matrices. While the matrices can be dense and large, computing matrix-vector
products can be done efficiently  by forward propagation. To this end, the right hand side 
of Eq.~\ref{bigsystem} is multiplied by the vector and the linear system is then solved
with a vector right hand side. The multiplication with the transpose, also known as the adjoint
method, is done by solving the equation backwards. It is important to note that the back-propagation algorithm
\cite{williams1986learning}
is nothing but a particular implementation of the adjoint method discussed much earlier in \cite{kalman1960contributions}.

\section{Regularization}
\label{sec:optim}
In this section we present derivative-based regularization functions and a multi-level learning approach  to ensure smoothness of parameters in deep learning. 
By biasing the learning process towards smooth time dynamics we aim at improving the stability of the forward propagation (e.g., to ensure that $\bfK(t)$ changes sufficiently slow as assumed in Sec.~\ref{sec:stability}) and ultimately the generalization.
Intuitively, we cannot expect the network to generalize well in the absence of smoothness.
While  the presented regularization techniques are commonly employed in other areas of inverse problems~\cite{hansen,vogelbook}, e.g., imaging their application to deep learning is, to the best of our knowledge, rather novel.

\subsection{Regularizing the Forward Propagation}\label{sub:timeReg}
The perhaps most common regularization strategy 
used in deep learning is referred to as \emph{ weight decay}, which is equivalent to standard Tikhonov regularization
of the form
$$ R(\bfK) = \hf \|\bfK\|^2_F,$$
where $\|\cdot\|_F$ denotes the Frobenius norm; see, e.g.,~\cite{GoodfellowEtAl2016}.
While this regularization reduces the magnitude of the weights, it is not sensitive to rapidly changing weights between adjacent layers.
To illustrate why this may be problematic, consider removing a single layer from a deep network. Since the network is deep, we should not expect large changes in the values of the output layer and thus similar classification errors. However, if this layer performs, e.g., a 90 degree rotation of the features while the adjacent layers keep the features unchanged, the effect will be dramatic. 

Our interpretation of forward propagation as a system of non-autonomous ODEs and the stability analysis in Sec. \ref{sec:stability} motivate that
$\bfK$ should be smooth, or at least piecewise smooth in time. To this end we propose the
following new regularization for the transformation weights
\begin{equation}
\label{regKb}
R(\bfK) = \frac{1}{2h} \sum \|\bfK_j - \bfK_{j-1} \|^2_F \; {\rm and}\; R(b) = \frac{1}{2h} \sum (b_j - b_{j-1} )^2.
\end{equation}
This regularization favors weights that vary smoothly between adjacent layers. Furthermore, as we see in our numerical experiments,
the regularization adds robustness to the process. We can easily add or subtract steps without significantly changing the final result, thus adding more generalizing power to our network.

\subsection{Regularizing the Classification Weights}

We propose using smoothness regularization on the 
classification weights for image classification problems, e.g., in Convolution Neural Networks (CNN); see; e.g.,~\cite[Ch.9]{GoodfellowEtAl2016}. For motivation, let the examples in $\bfY_0$ represent
vectorized $n_1 \times n_2 \times n_3$ images. The network propagates the images in time and generates perturbed
images of the same size.
The hypothesis function predicts the class label probabilities based on affinely transformed output images, i.e.,
$$ \bfC^{\rm pred} =  \bfh (\bfY_N \bfW +\bfe_s \mu^\top). $$
The number of rows in $\bfW$ is $n = n_1 \cdot n_2 \cdot n_3$ and the operation $\bfY_N(j,:)^{\top} \bfW(:,k)$
is a dot product between the $j$th output image and the classification weights for the $k$th class. Noting that the rows of $\bfY_N$ and the columns of $\bfW$ can be interpreted as discrete images sets the stage for developing regularization approaches commonly used in image processing; see, e.g.,~\cite{ChanShen2010}.

To further motivate the importance of regularization, note that if the number of examples is much larger than the number of features (pixels in the image) and if there is no significant redundancy, finding the optimal $\bfW$ given  $\bfY_N$ and $\bfC$ is an over-determined problem.
Otherwise the problem is ill-posed and there may be infinitely many weights that yield the same classification on the observed data. 
The risk in both cases is \emph{overfitting} since the optimal classification weights can be highly irregular and generalize poorly. Thus, one way to enforce uniqueness and improve generalization is to regularize the weights.

A standard approach in machine learning also known as \emph{pooling} aims at achieving this goal; see, e.g.,~\cite{GoodfellowEtAl2016}.
The main idea of pooling is to coarsen the output images and thereby to decrease the dimensionality of the classification problem.
The simplest approach is skipping, i.e., subsampling the image $\bfY_N$, e.g., at every other pixel.
Other common options are average pooling and max-pooling, where image patches are represented by their average and maximum value, respectively.
From an inverse problems perspective, pooling can be thought of as a subspace regularization~\cite{hansen}. For example, average pooling is similar to requiring $\bfW$ to be constant over each image patch.

An alternative approach to regularization is to interpret the $j$th feature, $\bfy_j \in \R^n$, and the $k$th classification weight, $\bfw_k \in \R^n$, of the CNN as discretizations of  image functions $y_j : \Omega \to \R$ and  $w_k : \Omega \to \R$, respectively, where $\Omega \subset \R^2$ is the image domain. Assuming $y_j$ is sufficiently regular (which is to be enforced by the regularization) we can see that the probability of the $j$th example belonging to the $k$th class can be viewed as
\begin{equation}
 \bfh(\bfy_j^{\top} \bfw_k + \mu_k) \approx \bfh \left( {\rm vol}(\Omega) \int_{\Omega} y(x) w(x)\, dx + \mu_k \right), 	
\end{equation}
where ${\rm vol} (\Omega)$ denotes the volume of the image domain.
To obtain weights that are insensitive to small displacements of the images it is reasonable
to favor spatially smooth parameters by using regularization of the form
\begin{equation}
\label{regW}
R(\bfw_k) = \hf \|\bfL \bfw_k \|^2,
\end{equation}
where $\bfL$ is a discretized differential operator. This regularization also embeds the optimal $w_k$ into a suitable function space.  For example, using an image gradient ensures that $w_k$ is in the Sobolev space $H^1(\Omega,\R)$~\cite{EvansPDE}. Most importantly for our application at hand, derivative-based regularization yields smooth classification weights that, as we see next, can be interpreted by visual inspection. 

\subsection{Multi-level Learning} 
\label{sub:multilevel_learning}
As another means of regularizing the problem, we exploit a multi-level learning strategy that gradually increases the number of layers in the network. Our idea is based on the continuous interpretation of the forward propagation in which the number of layers in the network corresponds to the number of discretization points. Our idea is closely related to cascadic multigrid methods~\cite{BornemannDeuflhard1996} and ideas in image processing where multi-level strategies are commonly used to decrease the risk of being trapped in local minima; see, e.g.,~\cite{ModSiamBook}. More details about multi-level methods in learning can be found in~\cite{haber2017learning}.

The basic idea is to first solve the learning problem using a network with only a few layers and then prolongate the estimated weights of the forward propagation to a network with twice as many layers. The prolongated weights are then used to initialize the optimization problem on the finer level. We repeat this process until a user-specified maximum number of layers is reached. 

Besides realizing some obvious computational savings arising from the reduced size of the networks, the main motivation behind our approach is to obtain good starting guesses for the next level. This is key since, while deeper architectures offer more flexibility to model complicated data-label relation, they are in our experience more difficult to initialize. Additionally, the Gauss-Newton method used to estimate the parameters of the forward propagation benefits in general from good starting guesses.

\section{Numerical Examples}
\label{sec:experiments}
In this section we present numerical examples for classification problems of varying level of difficulty. We begin with three examples aiming at learning classification functions in two variables, which allow us to easily assess and illustrate the performance of the DNNs. In Section~\ref{sub:MNIST} we show results for the MNIST data set~\cite{LeCunEtAl1990,mnistlecun2009}, which is a common benchmark problem in image classification. 

\subsection{Concentric Ellipses}
\label{sub:ell}
As a first test we consider a small-scale test problem in two dimensions. The test data consists of 1,200 points that are evenly divided into two groups that form concentric ellipsoids; see left subplot in Figure~\ref{fig:circles}. The original data is randomly divided into 1,000 training examples and 200 examples used for validation. 

We train the original ResNet, the antisymmetric ResNet, and the Hamiltonian network with Verlet propagation using the block coordinate descent method and a multi-level strategy with $4, 8, 16, 32, 64, 128, 256, 1024$ layers. Each block coordinate descent iteration consists of a classification step ($\leq$2 iterations of Newton-PCG with $\leq 2$ PCG iterations) and a Gauss-Newton-PCG step to update the propagation weights ($\leq$20 PCG iterations preconditioned by regularization operator). In all examples we use the logistic regression hypothesis function in Eq.~\ref{eq:logReg} and $\tanh$ activation function. The final time is $T=20$ and the width of the network is $n=2$. 
To enforce smoothness of the propagation weights in time we employ the regularizer in Eq.~\ref{regKb} weighted by a factor of $\alpha=10^{-3}$. No regularization is used in the classification. 

We show the performance of the multi-level scheme in the left subplot of Figure~\ref{fig:ml}. For this simple data set, all forward propagation methods achieve an optimal validation accuracy of $100\%$ at some level. As to be expected, the validation accuracy increases with increasing depth of the network, which also results in more degrees of freedom. 
The results for the Verlet method at the final level are shown in Figure~\ref{fig:circles}. The two steps of deep learning (propagation and classification) can be seen in the center plot. The propagation transforms the feature such that they can be linearly separated. The result of the learning process is a network that predicts the class for all points in the 2D plane, which is illustrated in the right subplot of Figure~\ref{fig:circles}. 

\begin{figure}[t]
	\begin{center}
		\includegraphics[width=.95\textwidth]{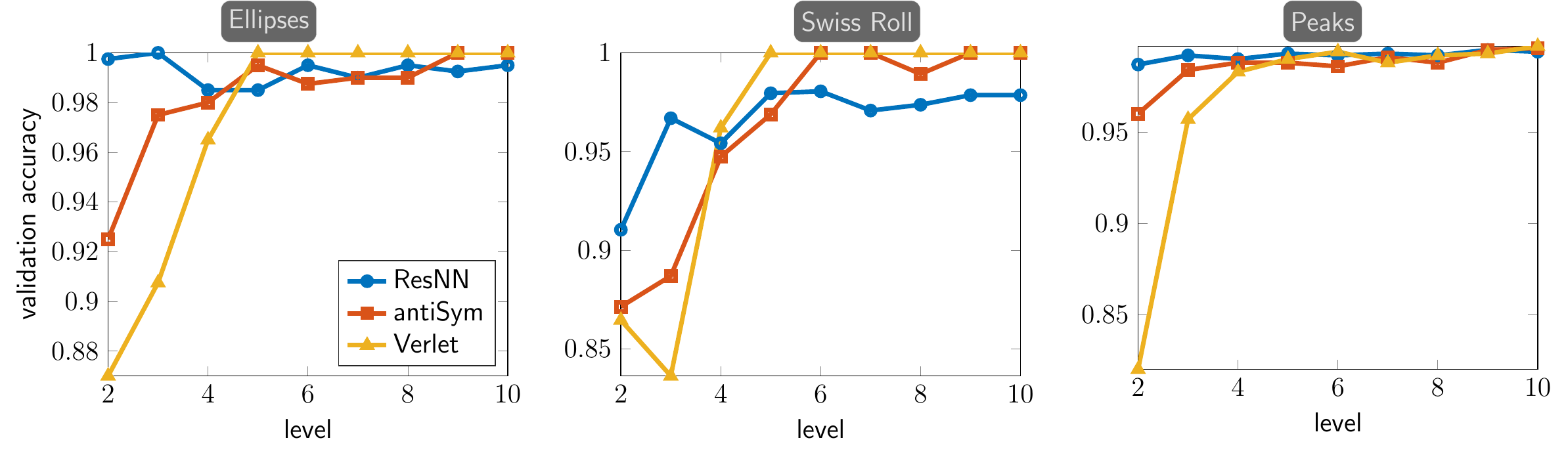}
	\end{center}
	\caption{Multi-level convergence for the two-dimensional examples in Sections~\ref{sub:ell}--\ref{sub:peaks}. For each level (level $\ell$ corresponds to DNN with $2^\ell$ layers) we show the best validation accuracy (the optimal value is $1$, which corresponds to a validation error of $0\%$).  }\label{fig:ml}
\end{figure}

\begin{figure}[t]
	\begin{center}
		\includegraphics[width=.95\textwidth]{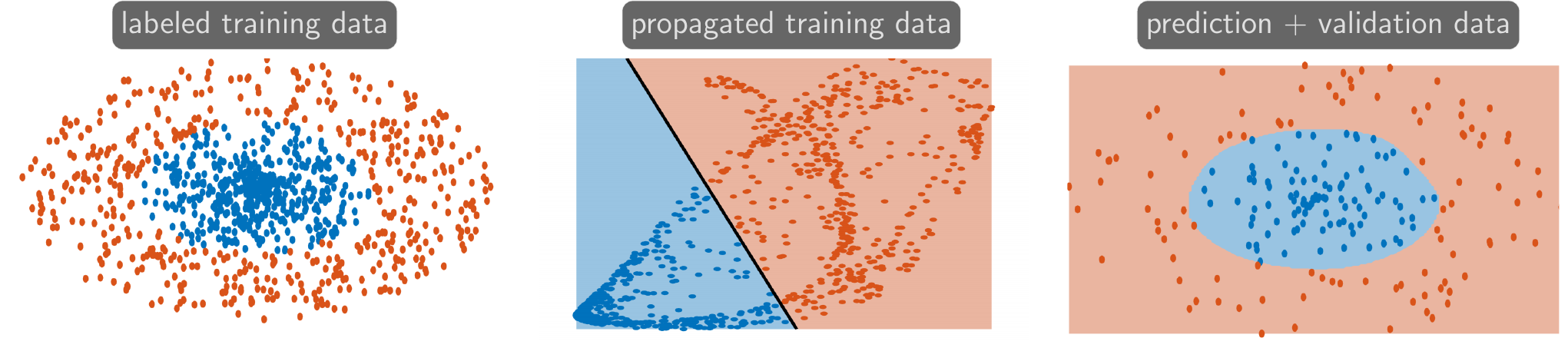}
	\end{center}
	\caption{Classification results for the ellipse problem from Section~\ref{sub:ell} using a Hamiltonian Neural Network with Verlet propagation with $N=1,024$ layers. Left: Labeled input features representing two concentric ellipses that are not linearly separable. Center: The output features are linearly separable and can be accurately classified. Right: We show the predictions of the network (colors in the background) superimposed by the validation data.  (This figure is optimized for screen use.)} 
	\label{fig:circles}
\end{figure}

\subsection{Swiss Roll}\label{sub:swiss}
We consider another small-scale test problem that is inspired by the swiss roll example. The data is obtained by sampling the vector functions
\begin{equation*}
	f_1(r,\theta) = r  \left(\begin{array}{r}
		  \cos(\theta) \\ \sin(\theta)
	\end{array}\right),
	\quad \text{ and } \quad
	f_2(r,\theta) = (r+0.2) \left(\begin{array}{r}
		   \cos(\theta) \\ \sin(\theta)
	\end{array}\right),
\end{equation*}
for $r\in[0,1]$ and $\theta \in [0,4\pi]$ at 513 points each. Every other point along the curve is removed from the data set and used for validation; see left subplot in Figure~\ref{fig:swiss}.

Given the data we use the same parameterization of the multi-level and block coordinate descent method as in the previous example. For all networks, except the Hamiltonian network with Verlet forward propagation, we increase the dimensionality of the input feature space to $n=4$. As before we use the $\tanh$ activation, logistic regression function in Eq.~\ref{eq:logReg}, and choose a final time of $T=20$. To enforce smoothness of the propagation weights in time we employ the regularizer in Eq.~\ref{regKb} weighted by a factor of $\alpha=5\cdot10^{-3}$. No regularization is used in the classification. 

We plot the validation accuracy for each network and the different steps of the multi-level strategy in the center subplot of Figure~\ref{fig:ml}. 
The standard ResNet and the Hamiltonian NN with Verlet forward propagation both achieve an optimal validation accuracy for $N=1,024$.  However, the convergence considerably faster for the Hamiltonian network that reaches the optimal accuracy with $N=32$ layers. 
 We visualize the results obtained using the Verlet method in Figure~\ref{fig:swiss}. 

\begin{figure}[t]
	\begin{center}
		\includegraphics[width=\textwidth]{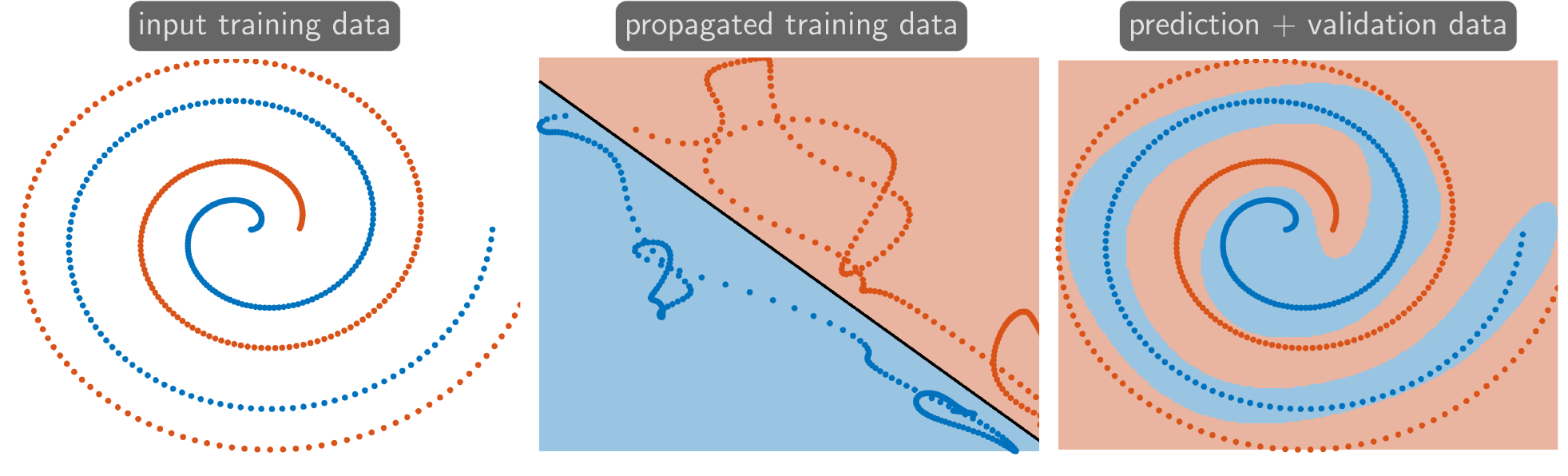}
	\end{center}
	\caption{Classification results for swiss roll example described in Section~\ref{sub:swiss} using a Hamiltonian network with $N=1,024$ layers and the Verlet forward propagation. Left: We show the training data (red and blue dots). Center: The Verlet method propagates the features so that they can be separated by the affine linear logistic regression classifier. Right: We show the interpolation obtained by our network (colored background) and the validation data. 
	(This figure is optimized for screen use.) }\label{fig:swiss}
\end{figure}

\subsection{Peaks}\label{sub:peaks}
We propose a new challenging test problem for classification into multiple classes using the \texttt{peaks} function in MATLAB$^\circledR$, which reads,
\begin{eqnarray*}
		f(x) = &  3 (1-x_1)^2 \exp(-(x_1^2) - (x_2+1)^2) - 10 (x_1/5 - x_1^3 - x_2^5) \\ 
		       &  \exp(-x_1^2-x_2^2) - 1/3 \exp(-(x_1+1)^2 - x_2^2),
\end{eqnarray*}
where $x \in [-3,3]^2$. The peaks function is smooth but has some nonlinearities and most importantly non-convex level sets. We discretize the function on a regular $256\times 256$ grid and divide the points into 5 different classes based on their function value. The points in each class are then randomly subsampled such that the training data approximately represents the volumes of the level sets. In our case the $s=5,000$ sample points are divided evenly into five classes of size 1,000. We illustrate the test data in the left subplot of Figure~\ref{fig:Peaks}.

We randomly choose 20\% of the data for validation and train the networks using the remaining examples. We use $\tanh$ as activation, the softmax hypothesis function in Eq.~\ref{eq:softmax},  and choose a final time of $T=5$. To enforce smoothness of the propagation weights in time we employ the regularizer in Eq.~\ref{regKb} weighted by a factor of $\alpha=5\cdot10^{-6}$. No regularization is used in the classification. We use the same multi-level strategy and identical parameters in the block coordinate descent methods  as in the previous examples.

We first train the standard ResNet and the antisymmetric variant where we use a width of $n=8$ by duplicating the features. The performance of both models is approximately the same; see right subplot in Figure~\ref{fig:ml}. The optimal accuracy, achieved for $N=1,024$, is around 98.8\%.

For the Hamiltonian network with Verlet forward propagation we use a narrower network containing only the original features (i.e., $n=2$ and no duplication). As in the previous experiments the accuracy generally improves with increasing depth (with one exception between levels 4 and 5). The optimal accuracy is obtained at the final level ($N=1,024$) with a validation error of $99.1\%$. We illustrate the results for the Verlet network in Figure~\ref{fig:Peaks}. The center subplot shows how the forward propagation successfully rearranges the features such that they can be labeled using a linear classifier. The right subplot shows that the prediction function fits the training data, but also approximates the true level sets.  

\begin{figure}[t]
	\begin{center}
		\includegraphics[width=\textwidth]{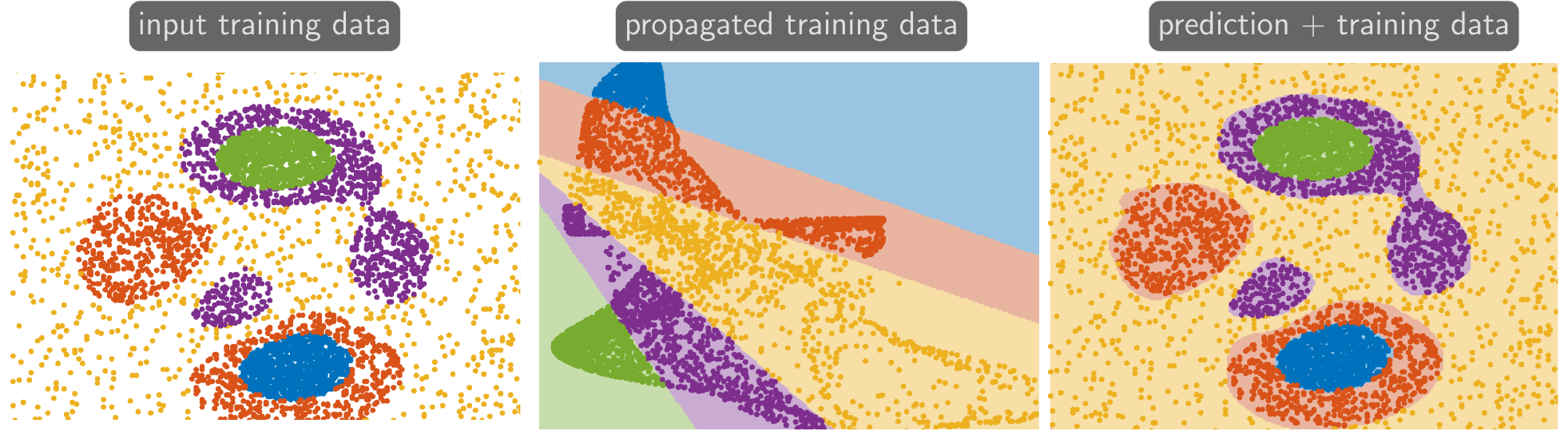}
	\end{center}
	\caption{Classification results for peaks example described in~\ref{sub:peaks} using a Hamiltonian network with $n=1,024$ layers and the Verlet forward propagation. Left: We illustrate the training data by colored dots that represent the class. Center: We show the propagated features and the predictions of the softmax classifier. Right: We depict the predictions of our network (colored background) and the training data. (This figure is optimized for screen use.)}\label{fig:Peaks}
\end{figure}

\subsection{MNIST}\label{sub:MNIST}

We use the MNIST dataset of hand-written digits~\cite{LeCunEtAl1990,mnistlecun2009} to illustrate the applicability of our methods in image classification problems.
The data set consists of 60,000 labeled digital images of size $28\times 28$ showing hand written digits from 0 to 9. We randomly divide the data into 50,000 training and 10,000 validation examples. We train the network using the standard and antisymmetric ResNet and the first-order Hamiltonian network using Verlet integration. 

We use a three-level multi-level strategy where the number of layers is $4, 8,$ and $16$. In each step we use the block coordinate descent method to approximately solve the learning problem and prolongate the forward propagation parameters to the next level using piecewise linear interpolation. The width of the network is 6 (yielding $n=4,704$ features used in the classification) and we use $3\times3$ convolution operators that are fully connected within a given layer to model the linear transformation matrices $\bfK_0, \ldots, \bfK_{N-1}$. The final time is set to $T=6$.  To compute the Gauss-Newton step we first compute the full gradient over all 50,000 examples and then randomly subsample 5,000 terms for Hessian computations in the PCG step. The maximum number of iterations is set to 20 at each layer.

Within each step of the block coordinate descent we solve the classification problem using at most 5 iterations of Newton-PCG with up to 10 inner iterations. We use a discretized Laplacian as a regularization operator in Eq.~\ref{regW} and its shifted symmetric product as a preconditioner to favor smooth modes; see~\cite{CalvettiSomersalo2005}. The regularization parameter used in the classification is $0.01$. The output values for all examples are used at the final layer and no pooling is performed. 

Smooth time dynamics are enforced by penalizing the time derivatives as outlined in Section~\ref{sub:timeReg} with a manually calibrated regularization parameter of $\alpha = 0.005$. The regularization operator is used to precondition the PCG iterations in the Gauss-Newton method. 

To show the performance of the multi-level strategy, we summarize the training and validation errors at the respective levels  in Table~\ref{tab:mnist}. The table shows similar performance for all forward propagation methods.  Note that both the validation and training error are reduced for larger number of layers, but no overfitting is observed.  In this experiment, the multi-level approach considerably simplified the initialization of the deep networks. For example, the initial parameters of the standard ResNet at the final level ($N=16$ layers) already gave a validation accuracy of 98.41\%, which was improved to 98.47\%. 
We illustrate the results of the antisymmetric ResNet, which yields slightly superior performance in terms of validation error, in Figure~\ref{fig:MNIST}. The smoothness of the classification weights enforced by our regularizer can be seen in the right column.

\begin{table}
	\begin{center}
		\begin{tabular}{|c|cc|cc|cc|}
			\hline
			       & \multicolumn{2}{|c|}{ResNet} & \multicolumn{2}{|c|}{anti symmetric ResNet} & \multicolumn{2}{|c|}{Hamiltonian Verlet} \\ \hline
			layers & TE      &  VE                & TE      &  VE                & TE      &  VE                \\ \hline
			4      &0.96\%   & 1.71\%            & 1.13\%  & 1.70\%            & 1.49\% &	2.29\%           \\
		    8 	   &0.80\%	 & 1.59\%            & 0.92\%	& 1.46\%             & 0.82\%   & 1.60\%            \\
		    16     &0.73\%   & 1.53\%             & 0.91\%	& 1.38\%           & 0.35\%   & 1.58\%                  \\ \hline
		\end{tabular}		
	\end{center}
	\caption{Multi-level training result for MNIST data set (see Section~\ref{sub:MNIST}). We show the Training Error (TE) and Validation Error (VE) for the standard and antisymmetric ResNet and the Hamiltonian inspired network with Verlet forward propagation for increasing number of layers in the network. The different forward propagation methods yield comparable results with the antisymmetric ResNet giving slightly lower validation errors at each level.}
	\label{tab:mnist}
\end{table}

\begin{figure}[t]
	\begin{center}
		\includegraphics[width=.9\textwidth]{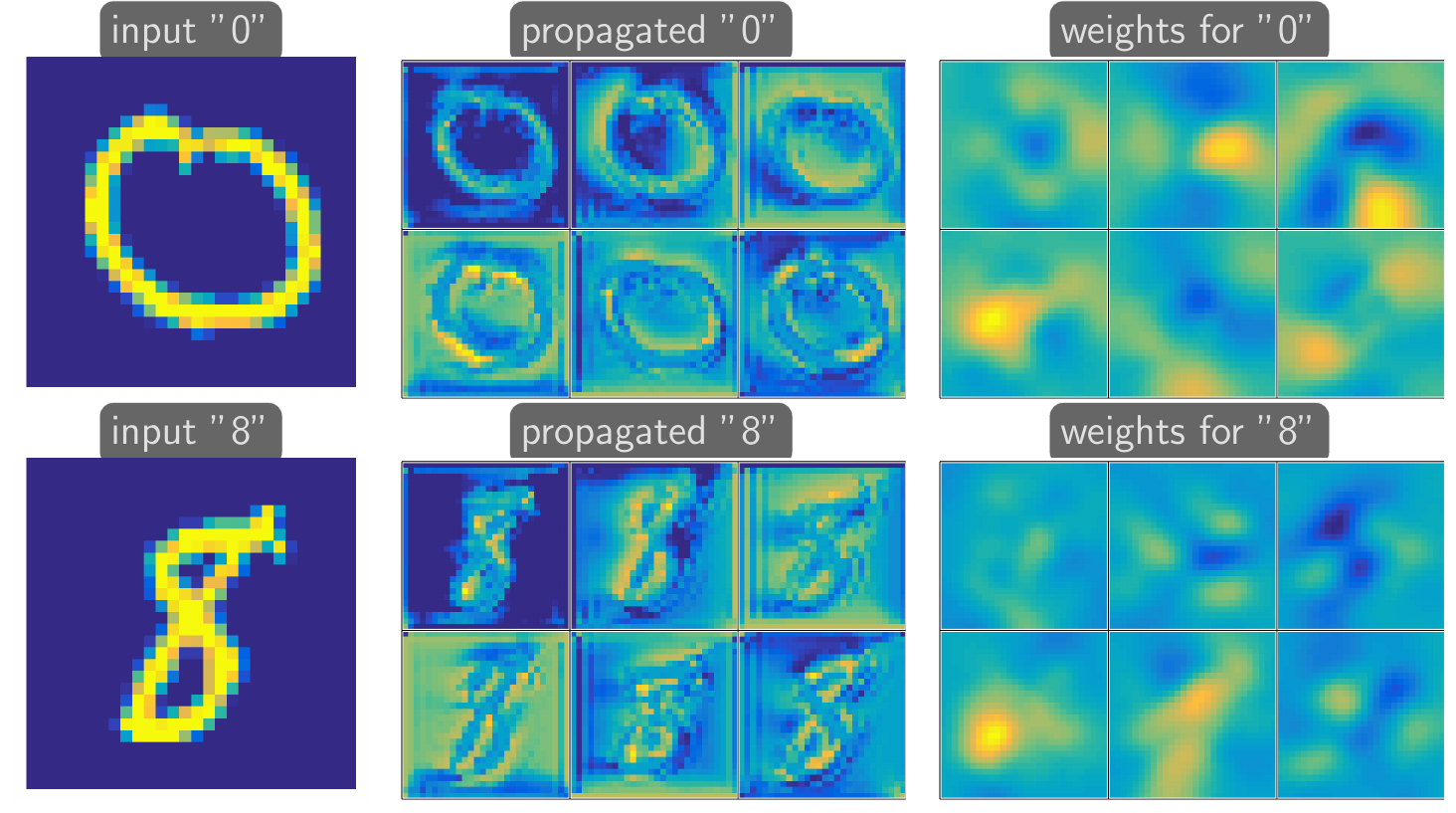}
	\end{center}
	\caption{Classification results for the MNIST data set  (see Section~\ref{sub:MNIST}) using an antisymmetric ResNet with $N=16$ layers and 6 neurons. We show two randomly chosen input images, their propagated versions at the output layer, and the classification weights for the two digits shown. The smoothness enforced by the second-order regularization operator is evident in the weights.} \label{fig:MNIST}
\end{figure}

\section{Summary and conclusions}
\label{sec:summary}
In this paper, we expose the relation between deep learning and dynamic inverse problems and lay the foundation for fruitful research at the interface of inverse problems and data science.
Specifically, we propose new architectures for deep neural networks (DNN) that improve the stability of the forward propagation. Using derivative-based regularization, we also improve the well-posedness of the learning task. 
We also propose a multi-level strategy that simplifies the choice of initial parameters and thereby simplifies the training of very deep networks.
Our forward propagation methods are inspired by Hamiltonian systems and motivated by a stability analysis that is based on a continuous formulation of the learning problem. 

Our stability discussion in Section~\ref{sec:stability} provides new insights why ResNets~\cite{he2016deep}, a commonly used forward propagation scheme in deep learning, can be unstable. Our result is based on a continuous interpretation of a simplified version of ResNets as a system of nonlinear ODEs and a standard stability argument. Using intuitive examples, we show the impact of the spectral properties of the transformation matrices on the stability of the forward propagation and state conditions under which gradients explode, vanish, or are stable. We also note that stability depends on the smoothness of the network parameters over time. Based on our findings and numerical experience, we argue that it is desirable to restrict ResNets to matrices that lead to Jacobians whose real parts of the spectrum are close to zero and penalize the time-derivative of network parameters through regularization. We also emphasize that the well-posedness of the learning problem in ResNets relies on sufficiently small time steps. Extending our analysis to more general versions of ResNet~\cite{he2016deep} is straightforward and will be investigated in future work.

Motivated by our theoretical considerations, we propose three new forward propagation methods in Section~\ref{sec:newArchitectures} that can be used to obtain well-posed learning problems for arbitrarily deep networks. The first one is a simple modification of ResNet that applies a linear transformation to the transformation matrices to annihilate the real parts of their eigenvalues. The other two methods are inspired by Hamiltonian systems and provide alternative ways to preserve in information in the forward propagation. The Hamiltonian networks require special integration techniques; see~\cite{ascherBook} for details on their treatment. The second-order forward propagation can be discretized using the leapfrog method, and our implementation of the first-order propagation uses the Verlet method. While the stability of the leapfrog method requires matrices with non-positive real eigenvalues, the Verlet method does not require any restrictions and yields the best performance in our numerical experiments.

Our approach to stabilizing the learning problem for very deep architectures differs significantly from existing approaches, most importantly \emph{batch normalization}~\cite{IoffeSzegedy2015}. Batch normalization scales values at hidden layers to prevent vanishing or exploding gradients and has been shown to improve the efficiency of training deep networks. In contrast to that our approach does not modify the values of the features, but alleviates the need of normalization by constructing stable forward propagation methods. 

In order to improve the generalization quality, improve stability, and simplify training of deep networks we also propose new regularization approaches that depend on our continuous formulation of the problem. We use derivative-based regularizers to favor smooth time dynamics and, for image classification problems, spatially smooth classification weights. To further regularize and simplify the initialization of the learning algorithm, we employ a multi-level learning strategy that gradually increases the depth of the network. In our experiments, this approach has been a simple yet effective way to obtain good initializations for the learning problems. Our regularization methods are commonly used in imaging science, e.g., image registration~\cite{ModSiamBook}, however, to the best of our knowledge not commonly employed in deep learning. Our numerical examples show that approximately solving the regularized learning problem yields works that generalize well even when the number of network parameters exceeds the number of training features.

We illustrate our methods using three academic test problems with available ground-truth and the MNIST problem~\cite{mnistlecun2009}, which is a commonly used benchmark problem for image classification. Our experiments show that the proposed new architectures yield results that are competitive with the established ResNet architecture. This is particularly noteworthy for the proposed anti-symmetric ResNet, where we restrict the dimensionality of the search space to ensure stability. 

By establishing a link between deep learning and dynamic inverse problems, we are positive that this work will stimulate further research by both communities. An important item of future work is investigating the impact of the proposed architectures on the performance of learning algorithms. Currently, stochastic gradient descent~\cite{RobbinsMonro1951} is commonly used to train deep neural networks due to its computational efficiency and empirical evidence supporting its generalizations~\cite{Bottou2012}. A specific question is if the improvements in the forward propagation and regularization proposed in this paper will lead to better generalization properties of subsampled second-order methods such as~\cite{ByrdEtAl2011,2016arXiv160104738R}. Another thrust of future work is the development of automatic parameter selection strategies for deep learning based on the approaches presented, e.g., in~\cite{ehn1,hansen,vogelbook, ho1, deSturlerKilmer2011}. A particular challenge in this application is the nontrivial relationship between the regularization parameters chosen for the classification and forward propagation parameters.

\section*{Acknowledgements}
This work is supported in part by National Science Foundation award DMS 1522599 and by the NVIDIA Corporation through their donation of a Titan X GPU.

\scriptsize

\end{document}